\documentclass[twoside]{article}
\usepackage{aistats2015}
\usepackage{hyperref}
\usepackage{url}

\usepackage{url}
\usepackage{color}
\usepackage{graphicx}
\usepackage{amssymb,amsfonts}
\usepackage{amsmath,amssymb}
\usepackage{float}
\usepackage{comment}
\usepackage{verbatim}
\usepackage{fullpage}
\usepackage{appendix,stfloats}
\usepackage{bbm}
\usepackage{blkarray}

\usepackage{wrapfig}

\usepackage{amsmath,amssymb,amsthm,graphicx,url,mathtools}
\usepackage[ruled]{algorithm}

\usepackage[noend]{algpseudocode}

\makeatletter
\def\BState{\State\hskip-\ALG@thistlm}
\makeatother

\usepackage{hyperref}
\usepackage{enumerate}
\usepackage{subfigure,color}
\usepackage{array,dsfont}

\usepackage{wrapfig}

\newcommand{\argmin}{\operatornamewithlimits{argmin}}

\newtheorem{theorem}{Theorem}[section]
\newtheorem{lemma}[theorem]{Lemma}

\newtheorem{corollary}[theorem]{Corollary}

\newenvironment{definition}[1][Definition]{\begin{trivlist}
\item[\hskip \labelsep {\bfseries #1}]}{\end{trivlist}}
\begin{document}

%

%

\twocolumn[

\aistatstitle{Max-Cost Discrete Function Evaluation Problem under a Budget}

\aistatsauthor{ Feng Nan \And Joseph Wang \And Venkatesh Saligrama }
\aistatsaddress{ Boston University \\fnan@bu.edu \And Boston University \\ joewang@bu.edu \And Boston University \\ srv@bu.edu } ]

\begin{abstract}
We propose novel methods for max-cost Discrete Function Evaluation Problem (DFEP) under budget constraints. We are motivated by applications such as clinical diagnosis where a patient is subjected to a sequence of (possibly expensive) tests before a decision is made. Our goal is to develop strategies for minimizing max-costs. The problem is known to be NP hard and greedy methods based on specialized impurity functions have been proposed. We develop a broad class of \emph{admissible} impurity functions that admit monomials, classes of polynomials, and hinge-loss functions that allow for flexible impurity design with provably optimal approximation bounds. This flexibility is important for datasets when max-cost can be overly sensitive to ``outliers.'' Outliers bias max-cost to a few examples that require a large number of tests for classification. We design admissible functions that allow for accuracy-cost trade-off and result in $O(\log n)$ guarantees of the optimal cost among trees with corresponding classification accuracy levels.  
\end{abstract}

\section{Introduction}

In many applications such as clinical diagnosis, monitoring, and web search, a patient, entity or query is subjected to a sequence of tests before a decision or prediction is made. Tests can be expensive and often complementary, namely, the outcome of one test may render another redundant. The goal in these scenarios is to minimize total test costs with negligible loss in diagnostic performance.

We propose to formulate this problem as an instance of the Discrete Function Evaluation Problem (DFEP). Under this framework, we seek to learn a decision tree which {\it correctly} classifies data while minimizing the cost of testing. We then propose methods to {\it trade-off} accuracy and costs.
%
%
%
%
%

An instance of the problem is defined as $I=(S,C,T,\mathbf{c})$; Here $S=\{s_1,\dots,s_n\}$ is the set of $n$ objects; $C=\{C_1,\dots,C_m\}$ is a partition of $S$ into $m$ classes; $T$ is a set of tests; 
$\mathbf{c}$ is a cost function that assigns a cost $c(t)\geq 0$ for each test $t\in T$. Applying test $t\in T$ on object $s\in S$ will output a discrete value $t(s)$ in a finite set of possible outcomes $\{1,\dots, l_t\}$. $T$ is assumed to be complete in the sense that for any distinct $s_i,s_j\in S$ there exists a $t\in T$ such that $t(s_i)\neq t(s_j)$ so they can be distinguished by $t$. Given an instance of the DFEP, the goal is to build a testing procedure that uses tests in $T$ to determine the class of an unknown object. 
Formally, any testing procedure can be represented by a decision tree, where every internal node is associated with a test and objects are directed from the root to the corresponding leaves based on the test outcomes at each node. Given instance $I$ and decision tree $D$, the testing cost of $s\in S$, denoted as $cost(D,s)$, is the sum of all costs incurred along the root-to-leaf path in $D$ traced by $s$. We define the total cost as
$$Cost_W(D)=\underset{s\in S}{\max}  cost(D,s)$$
This is known as the max-cost testing problem in the DFEP literature and has independently received significant attention~\cite{DiagnosisDeterminationSimultaneous,TradingOffWorstExpectedCost,MoshkovGreedyAlgorithmwithWeightsforDecisionTreeConstruction,GroupBasedActiveLearning} due to the fact that in real world problems, the prior probability used to compute the expected testing cost is either unavailable or inaccurate. Another motivation stems from time-critical applications, such as emergency response \cite{GroupBasedActiveLearning}, where violation of a time-constraint may lead to unacceptable consequences.

In this paper we propose novel approaches and themes for the max-cost DFEP problem.  It is now well-known~\cite{DiagnosisDeterminationSimultaneous} that $O(\log n)$ is the best approximation factor for DFEP unless $P=NP$. Greedy methods that achieve $O(\log n)$ approximation guarantee have been proposed~\cite{DiagnosisDeterminationSimultaneous,TradingOffWorstExpectedCost,MoshkovGreedyAlgorithmwithWeightsforDecisionTreeConstruction}. These methods often rely on {\it judiciously} engineering so called impurity functions that are surprisingly effective in realizing ``optimal'' $O(\log n)$ guarantees. Authors in
\cite{DiagnosisDeterminationSimultaneous,MoshkovGreedyAlgorithmwithWeightsforDecisionTreeConstruction,TradingOffWorstExpectedCost} describe impurity functions based on the notion of Pairs, while the authors in \cite{GroupBasedActiveLearning} describe more complex impurity functions but require distributional assumptions.

In contrast, we propose a broad class of \emph{admissible} functions such that any function from this class can be chosen as an impurity function with an $O(\log n)$ approximation guarantee. Our admissible functions are in essence positive, monotone supermodular functions and admit not only pairs, monomials, classes of polynomials, but also hinge-loss functions. 

We propose new directions for the max-cost DFEP problem. In contrast to the current emphasis on correct classification, we propose to deliberately trade-off cost with accuracy. This perspective can be justified under various scenarios. First, max-cost is overly sensitive to ``{\it outliers},'' namely, a few instances require prohibitively many tests for correct classification. In these situations max-cost is not representative of most of the data and is biased towards a small subset of objects. Consequently, censoring those few ``outliers'' is meaningful from the perspective that max-cost applies to {\it all but few} examples. Second many applications have hard cost constraints that supersede correct classification of the entire data set and the goal is a tree that guarantees these cost constraints while minimizing errors. 

Our proposed admissible functions are sufficiently general and allows for trading accuracy for cost. In particular we develop methods with $O(\log n)$ guarantees of the optimal cost among trees with a corresponding classification accuracy level. Moreover, we show empirically on a number of examples that selection of impurity functions plays an important role in this trade-off. In particular some admissible functions, such as hinge-loss are particularly well-suited for low-budgets while others are preferable in high-budget scenarios.

Apart from the related approaches already described above, our work is also related to those that generally deal with expected costs~\cite{AdaSubmodular_jair2011,NearOptimalBayesianActiveLearning,GroupBasedActiveLearning} or related problems such as sub-modular set coverage problem \cite{InteractiveSubmodularSetCover_GuilloryB10}. At a conceptual level the main difference in ~\cite{InteractiveSubmodularSetCover_GuilloryB10,AdaSubmodular_jair2011,NearOptimalBayesianActiveLearning,GroupBasedActiveLearning} is in the way tests are chosen. Unlike our approach these methods employ utility functions in the policy space that acts on a sequence of observations. \cite{AdaSubmodular_jair2011} develops the notion of adaptive submodularity and has applied it for automated diagnosis. The proposed adaptive greedy algorithm can handle multiple classes/ test outcomes and arbitrary test costs but the approximation factor for the max-cost depends on the prior probability and can be very large in adversarial situations.
%
%
A popular class of related approximation algorithms is \emph{generalized binary search} (GBS) \cite{Dasgupta04analysisof,KosarajuOnanOptimalSplitTreeProblem,Nowak08generalizedbinary}. 
A special case of this problem is where each object belongs to a distinct class and is known as \emph{object identification problem} \cite{DecisionTreesforEntityIdentification} or pool-based active learning \cite{Dasgupta04analysisof}. 
When tests are restricted to binary outcomes and uniform test costs, $O(\log(1/p_{min}))$ approximation, where $p_{min}$ is the minimum probability of any single object \cite{Dasgupta04analysisof} can be obtained. Alternatively  \cite{Gupta:OptimalDecisionTreesandAdaptiveTSP} provides an algorithm which leads to an $O(\log n)$ approximation factor  for the optimal expected cost with arbitrary test costs and binary test outcomes. With respect to the max-cost, \cite{Hanneke06thecost} gave a $O(\log n)$ approximation for multiway tests and arbitrary test costs.

\paragraph{Organization:}
We present a greedy algorithm in Section \ref{sec:greedy_analysis} which we show under general assumptions on the impurity function leads to an $O(\log n)$ approximation of the optimal tree. We examine the assumptions on impurity functions and use them to define a class of \emph{admissible} impurity functions in Section \ref{sec:admissible}. Following this, we generalize from the error-free case to trade-off between max-cost and error in Section \ref{sec:trade-off}. Finally, we demonstrate performance of the greedy algorithm on real world data sets in Section \ref{sec:experiments} and show the advantage of different impurity functions along with the trade-off between error and max-cost.

\section{Greedy Algorithm and Analysis}\label{sec:greedy_analysis}
In this section, we present an analysis of the greedy algorithm \textsc{GreedyTree}. We first show that \textsc{GreedyTree} yields a tree whose max-cost is within $O(\log n)$ of the optimal max-cost for any DFEP. This bound on max-cost holds for any impurity function that satisfies a very general criteria as opposed to a fixed impurity function. In Section ~\ref{sec:admissible} we examine the assumptions on the impurity functions and present multiple examples of impurity functions for which this approximation bound holds.



Before beginning the analysis, we first define the following terms: for a given impurity function $F$, $F(G)$ is the impurity function on the set of objects $G$; $D_{F}$ is the family of decision trees with $F(L)=0$ for any of its leaf $L$; $OPT(S)$ is the minimum max-cost among all trees in $D_{F}$ for the given input set of objects $S$; $Cost_{F}(S)$ is the max-cost of the tree constructed by {\textsc{GreedyTree}} based on impurity function $F$.

\begin{algorithm}
\caption{{\textsc{GreedyTree}}}\label{algo:GreedyTree}
\begin{algorithmic}[1]
\Procedure{GreedyTree(G,T)}{}
\If {$F(G)=0$} \Return
\EndIf
\For {each test $t\in T$}
\State Compute
$R(t):= \underset{i\in \text{outcomes}}{\max} \frac{c(t)}{F(G)-F(G^i_t)}$,
\State where $G^i_t$ is the set of objects in $G$
\State that has outcome $i$ for test $t$.
\EndFor
\State $\hat{t} \gets \argmin_t R(t)$
\State $T \gets T\backslash \{\hat{t}\}$
\For {each outcome $i$ of $\hat{t}$}
\State $\textsc{GreedyTree}(G^i_t,T)$
\EndFor
\EndProcedure
\end{algorithmic}
\end{algorithm}
For simplicity, we assume the impurity function takes on integer values and outcome-independent test costs. Note that integer valued impurity functions is not a limitation because of the discrete (finite) nature of the problem - one can always scale any rational-valued impurity function to make it integer-valued. Similarly, it can be easily shown that our result extends to the outcome-dependent cost setting considered in \cite{TradingOffWorstExpectedCost} as well.

Given a DFEP, {\textsc{GreedyTree}} greedily chooses the test with the largest worst-case impurity reduction until all leaves are pure, i.e. impurity equals zero.
Let $\tau$ be the first test selected by {\textsc{GreedyTree}}. By definition of the max-cost,
\begin{equation*}
\frac{Cost_{F}(S)}{OPT(S)}=\frac{c(\tau)+\underset{i}{\max} Cost_{F}(S^i_\tau)}{OPT(S)},
\end{equation*}
where $S^i_\tau$ is the set of objects in $S$ that has outcome $i$ for test $\tau$.
Let $q$ be such that $Cost_F(S^q_\tau)=\underset{i}{\max} Cost_{F}(S^i_\tau)$. We first provide a lemma to lower bound the optimal cost, which will later be used to prove a bound on the cost of the tree.

\begin{lemma}\label{lemma:OPT_Wlowerbound}
Let $F$ be monotone and supermodular, and $\tau$ is the first test chosen by {\textsc{GreedyTree}} on the set of objects $S$, then
\begin{equation*}
c(\tau)F(S)/(F(S)-F(S^q_\tau)) \leq OPT(S).
\end{equation*}
\end{lemma}
\vspace{-.5cm}
\begin{proof}
Let $D^*\in D_{F}$ be a tree with optimal max-cost. Let $v$ be an arbitrarily chosen internal node in $D^*$, let $\gamma$ be the test associated with $v$ and let $R\subseteq S$ be the set of objects associated with the leaves of the subtree rooted at $v$. Let $i$ be such that $c(\tau)/(F(S)-F(S^i_\tau))$ is maximized and $j$ be such that $c(\gamma)/(F(S)-F(S^i_\gamma))$ is maximized. We then have:
\begin{align}
&\frac{c(\tau)}{F(S)-F(S^q_\tau)} \leq \frac{c(\tau)}{F(S)-F(S^i_\tau)} \notag \\
&\leq
\frac{c(\gamma)}{F(S)-F(S^j_\gamma)} \leq \frac{c(\gamma)}{F(R)-F(R^j_\gamma)}.\label{eq:lemma2}
\end{align}
The first inequality follows from the definition of $i$. The second inequality follows from the greedy choice at the root. To show the last inequality, we have to show $F(S)-F(S^j_\gamma)\geq F(R)-F(R^j_\gamma)$. This follows from the fact that $S^j_\gamma \cup R \subseteq S$ and $R^j_\gamma = S^j_\gamma \cap R$ and therefore $F(S)\geq F(S^j_\gamma \cup R) \geq F(S^j_\gamma)+F(R)-F(R^j_\gamma)$, where the first inequality follows from monotonicity and the second follows from the definition of supermodularity.

For a node $v$, let $S(v)$ be the set of objects associated with the leaves of the subtree rooted at $v$. Let $v_1,v_2,\dots,v_p$ be a root-to-leaf path on $D^*$ as follows: $v_1$ is the root of the tree, and for each $i=1,\dots,p-1$ the node $v_{i+1}$ is a child of $v_i$ associated with the branch of $j$ that maximizes $c(t_i)/(F(S)-F(S^j_{t_i}))$, where $t_i$ is the test associated with $v_i$. If follows from \eqref{eq:lemma2} that
\begin{equation}
\frac{[F(S(v_i))-F(S(v_{i+1}))]c(\tau)}{F(S)-F(S^q_\tau)}\leq c_{t_i}.
\end{equation}
Since the cost of the path from $v_1$ to $v_p$ is no larger than the max-cost of the $D^*$, we have that
\vspace{-.3cm}
\begin{align*}
& OPT(S) \geq \sum_{i=1}^{p-1}c_{t_i} \\
 & \geq \frac{c(\tau)}{F(S)-F(S^q_\tau)}\sum_{i=1}^{p-1}(F(S(v_i))-F(S(v_{i+1}))\\
&=\frac{c(\tau)(F(S)-F(S(v_p))}{F(S)-F(S^q_\tau)}=\frac{c(\tau)F(S)}{F(S)-F(S^q_\tau)}.
\end{align*}
\end{proof}
\vspace{-1.5cm}
Using Lemma \ref{lemma:OPT_Wlowerbound}, we can now state the main theorem of this section which bounds the cost of the greedily constructed tree.
\begin{theorem} \label{thm:logn}
{\textsc{GreedyTree}} constructs a decision tree achieving $O(\log n)$-factor approximation of the optimal max-cost in $D_{F}$ on the set $S$ of $n$ objects if  $F$ is non-negative, monotone, supermodular with $\log(F(S))=O(\log n)$.
\end{theorem}
\begin{proof}
\begin{align}
& \frac{Cost_{F}(S)}{OPT(S)} =\frac{c(\tau)+Cost_{F}(S^q_\tau)}{OPT(S)} \\
&\leq \frac{c(\tau)}{OPT(S)}+\frac{Cost_{F}(S^q_\tau)}{OPT(S^q_\tau)} \label{eq:thm1_1}
\\ & \leq \frac{F(S)-F(S^q_\tau)}{F(S)}+\frac{Cost_{F}(S^q_\tau)}{OPT(S^q_\tau)} \label{eq:thm1_2} \\
& \leq \log (\frac{F(S)}{F(S_\tau^q)})+\log (F(S_\tau^q))+1 \label{eq:thm1_4}\\
& = \log (F(S))+1 =O(\log n). \label{eq:thm1_5}
\end{align}
The inequality in \eqref{eq:thm1_1} follows from the fact that $OPT(S) \geq OPT(S^q_\tau)$. \eqref{eq:thm1_2} follows from Lemma \ref{lemma:OPT_Wlowerbound}. The first term in \eqref{eq:thm1_4} follows from the inequality $\frac{x}{x+1} \leq \log(1+x)$ for $x>-1$ and the second term follows from the induction hypothesis that for each $G\subset S$, ${Cost_{F}(G)}/{OPT(G)}\leq \log(F(G))+1$. If $F(G)=0$ for some set of objects $G$, we define ${Cost_{F}(G)}/{OPT(G)}=1$.

We can verify the base case of the induction as follows.
if $F(G)=\beta$, which is the smallest non-zero impurity of $F$ on subsets of objects $S$, we claim that the optimal decision tree chooses the test with the smallest cost among those that can reduce the impurity function $F$:
\begin{equation*}
OPT(G)=\min_{t|F(G^i_t)=0, \forall i\in \text{outcomes}} c(t).
\end{equation*}
Suppose otherwise, the optimal tree chooses first a test $t$ with a child node $G'$ such that $F(G')=\beta$ and later chooses another test $t'$ such that all the child nodes of $G'$ by $t'$ has zero impurity, then $t'$ could have been chosen in the first place to reduce all child nodes of $G$ to zero impurity by supermodularity of $F$ and therefore this cannot be the optimal ordering of tests.
On the other hand, $R(t)=\infty$ in {\textsc{GreedyTree}} for those test $t$ that cannot reduce impurity and $R(t)=c(t)$ for those tests that can. So the algorithm would pick the test among those that can reduce impurity and have the smallest cost. Thus, we have shown that ${Cost_{F}(G)}/{OPT(G)} = \log(F(G))+1=1$ for the base case.
\end{proof}

Given that $P\neq NP$, the optimal order approximation for the DFEP problem is $O(\log n)$, which is achieved by \textsc{GreedyTree}. This approximation is not dependent on a particular impurity function, but instead holds for any function which satisfies the assumptions. In Section \ref{sec:admissible}, we define a family of impurity functions that satisfy these assumptions.


\section{\emph{Admissible} Functions}\label{sec:admissible}
A fundamental element of constructing decision trees is the impurity function, which measures the disagreement of labels between a set of objects. Many impurity functions have been proposed for constructing decision trees, and the choice of impurity function can have a significant impact on the performance of the tree. In this section we examine the assumptions placed on the impurity function by Lemma \ref{lemma:OPT_Wlowerbound} and Theorem \ref{thm:logn} which we use to define a class of functions we call \emph{admissible} impurity functions and provide examples of \emph{admissible} impurity functions.
\begin{definition} 
A function $F$ of a set of objects is \emph{admissible} if it satisfies the following five properties: (1) Non-negativity: $F(G)\geq 0$ for any set of objects $G$; (2) Purity: $F(G)=0$ if $G$ consists of objects of the same class; (3) Monotonicity: $F(G)\geq  F(R), \forall R \subseteq G$;  (4) Supermodularty: $F(G\cup j)-F(G)\geq F(R\cup j) -F(R)$ for any $R\subseteq G$ and object $j\notin R$; (5) $\log(F(S))=O(\log n)$.
\end{definition}
A wide range of functions falls into the class of \emph{admissible} impurity functions. We propose a general family of polynomial functions which we show is \emph{admissible}. Given a set of objects $G$, $n_G^i$ denotes the number of objects in $G$ that belong to class $i$.
\begin{lemma} \label{lemma:polynomialAdmissible}
Suppose there are $k$ classes in $G$. Any polynomial function of $n_G^1,\dots,n_G^k$ with non-negative terms such that $n_G^1,\dots,n_G^k$ do not appear as singleton terms is \emph{admissible}. Formally, if
\begin{equation}\label{eq:poly}
F(G)=\sum_{i=1}^{M} \gamma_i (n_G^1)^{p_{i1}}(n_G^2)^{p_{i2}}\dots (n_G^k)^{p_{ik}},
\end{equation}
where $\gamma_i$'s are non-negative, $p_{ij}$'s are non-negative integers and for each $i$ there exists at least 2 non-zero $p_{ij}$'s, then $F$ is \emph{admissible}.
\end{lemma}
\begin{proof}
Properties (1),(2),(3) and (5) are obviously true. To show $F$ is supermodular, suppose $R\subset G$ and object $\hat{j} \notin R$ and $\hat{j}$ belongs to class $j$, we have
\begin{align*}
& F(R\cup \hat{j})-F(R)\\
& =\sum_{i\in I_j} \gamma_i [(n_R^1)^{p_{i1}}\dots(n_R^j+1)^{p_{ij}}\dots (n_R^k)^{p_{ik}}-\\
& \qquad (n_R^1)^{p_{i1}}\dots(n_R^j)^{p_{ij}}\dots (n_R^k)^{p_{ik}}]\\
& \leq \sum_{i\in I_j} \gamma_i [(n_G^1)^{p_{i1}}\dots(n_G^j+1)^{p_{ij}}\dots (n_G^k)^{p_{ik}}-\\
& \qquad (n_G^1)^{p_{i1}}\dots(n_G^j)^{p_{ij}}\dots (n_G^k)^{p_{ik}}]\\
&=F(G\cup \hat{j})-F(G),
\end{align*}
where the first summation index set $I_j$ is the set of terms that involve $n_R^j$. The inequality follows because $(n_R^j+1)^{p_{ij}}$ can be expanded so the negative term can be canceled, leaving a sum-of-products form for $R$, which is term-by-term dominated by that of $G$.
\end{proof}

A special case of polynomial impurity function is the previously proposed Pairs function $P(G)$ \cite{TradingOffWorstExpectedCost,DiagnosisDeterminationSimultaneous,MoshkovGreedyAlgorithmwithWeightsforDecisionTreeConstruction}. Two objects $(s_1,s_2)$ are defined as a pair if they are of different classes, with the Pairs function $P(G)$ equal to the total number of pairs in the set $G$:
\vspace{-.2cm}
\begin{equation*}
P(G)=\sum_{i=1}^{k-1}\sum_{j=i+1}^{k} n_G^in_G^j,
\end{equation*}
where $k$ is the number of distinct classes in set $G$.
\begin{corollary}\label{cor:pairsAdmissible}
The Pairs impurity function is \emph{admissible}.
\end{corollary}
As a corollary of Theorem \ref{thm:logn} and Corollary \ref{cor:pairsAdmissible}, we see that $O(\log n)$ approximation for Pairs and outcome-dependent cost holds for multiple test outcomes as well, extending the binary outcome setting shown in \cite{TradingOffWorstExpectedCost}.

Another family of \emph{admissible} impurity functions is the Powers function.
\begin{corollary}\label{cor:powerW}
Powers function
\begin{equation}\label{eq:powerFunc}
F(G)=(\sum_{i=1}^{k} n_G^i)^l - \sum_{i=1}^{k}(n_G^i)^l
\end{equation}
is \emph{admissible} for $l=2,3,\dots$.
\end{corollary}
Note Pairs can be viewed as a special case of Powers function when $l=2$.
An important property of the Powers impurity functions is the fact that for any power $l$, the function is zero only if the set of objects all belong to the same class. As a result, using any of these Powers impurity function in {\textsc{GreedyTree}} results in an \emph{error-free} tree with near optimal cost.

Another interesting \emph{admissible} impurity used in Section \ref{sec:trade-off} is the hinged-Pairs function defined:
\vspace{-.1cm}
\begin{equation}\label{eq:hingedPairs}
P_\alpha(G)=\sum_{i\neq j}[[n^i_G-\alpha]_+[n^j_G-\alpha]_+-\alpha^2]_+,
\end{equation}
where $[x]_+=\max(x,0)$. This function differs from the Powers impurity function due to the fact that for a $\alpha>0$, the function $P_\alpha(G)=0$ need not imply that all objects in $G$ belong to the same class. In the next section, we will discuss how this allows for trees to be constructed incorporating classification error.
We include the proof of the following lemma in the Appendix.
\begin{lemma}\label{lemma:F_admissible_multi}
In the multi-class setting, $P_\alpha(G)$ is \emph{admissible}.
\end{lemma}

\textbf{Impurity Function Selection:} While all \emph{admissible} impurity functions enjoy the $O(\log n)$ approximation of the optimal max-cost, they lead to different trees depending on the problems. To illustrate this point, consider the toy example in Figure \ref{fig:demo}. A set $G$ has 30 objects in class 1 (circles) and 30 objects in Class 2 (triangles). Two tests $t_1$ and $t_2$ are available to the algorithm. Test $t_1$ separates 20 objects of Class 2 from the rest of the objects while $t_2$ evenly divides the objects into halves with equal number of objects from Class 1 and Class 2 in either half. Intuitively, $t_2$ is not a useful test from a classification point of view because it does not separate objects based on class at all. This is reflected in the right plot of Figure \ref{fig:demo}: choosing $t_2$ increases cost but does not reduce classification error while choosing $t_1$ reduces the error to $\frac{1}{6}$. If the impurity function chosen is the Pairs function, test $t_2$ will be chosen due to the fact that Pairs biases towards tests with balanced test outcomes. In contrast, the hinged-Pairs function leads to test $t_1$, and therefore may be preferable in this case (for more details on this example see the Appendix). Although both impurity functions are \emph{admissible} and return trees with near optimal guarantees, empirical performance can differ greatly and is strongly dependent on the structure of the data. In practice, we find that choosing the tree with the lowest classification error across a variety of impurity functions yields improved performance compared to a single impurity function strategy.

\begin{figure}
\centering
\includegraphics[trim=2.5cm 6cm 2.5cm 3cm,angle=0,width=0.4\textwidth]{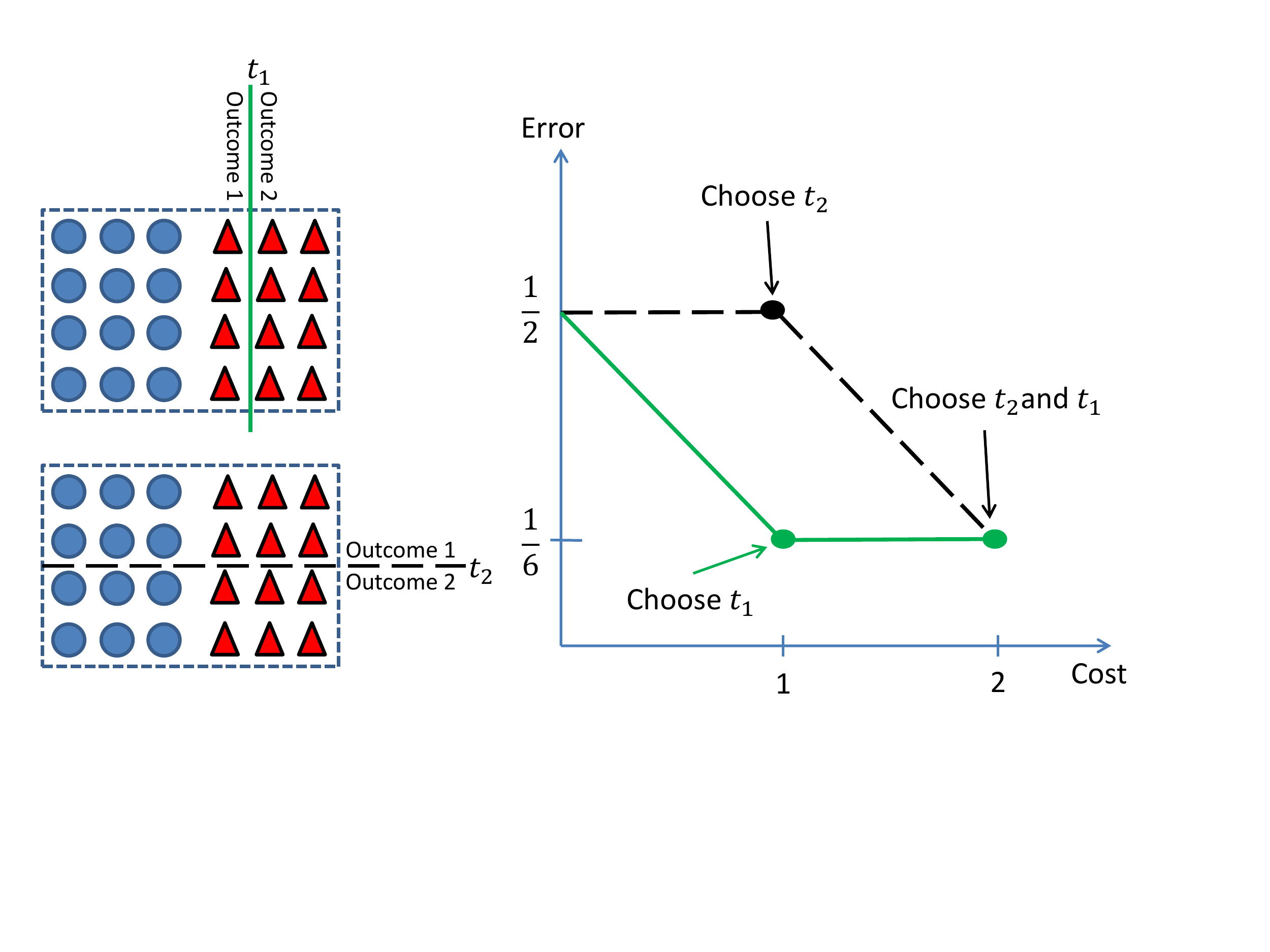}
\caption{Illustration of different impurity functions for different greedy choice of tests. The left two figures above show the test outcomes of test $t_1$ and $t_2$. The right figure shows the classification error against cost (number of tests). Here using Pairs leads to choosing $t_2$ because it prefers balanced splits; using the hinged-Pairs leads to choosing $t_1$, which is better from an error-cost trade-off point of view.} \label{fig:demo}
\end{figure}

\section{Trade-off Bounds} \label{sec:trade-off}
Up to this point, we have focused on constructing \emph{error-free} trees. Unfortunately, the max-cost criteria is highly sensitive to outliers, and therefore often yields trees with unnecessarily large maximum depth to accommodate a small subset of outliers in the data set. Refer to the synthetic experiment in Section \ref*{sec:experiments} for such an example. To overcome the sensitivity to outliers, we present an approach to constructing near optimal trees with non-zero error rates.


\paragraph{Early-stopping:}
Instead of requiring all leaves to have zero impurity ($F(L)=0$) in Algorithm \ref{algo:GreedyTree}, we can stop the recursion as soon as all leaves have impurity below a threshold $\delta$ ($F(L)\leq \delta$). This will allow error and cost trade-off. Let $D_{F:\delta}$ denote the set of trees with $F(L)\leq \delta$ for all leaves $L$ and let $OPT_{F:\delta}(S)$  denote the optimal max-cost among all trees in $D_{F:\delta}$.

Similar to the \emph{error-free} setting, the $O(\log n)$ approximation of the optimal max-cost still holds for early stopping as shown next. The proofs of Lemma \ref{lemma:OPT_WlowerboundEta} and Theorem \ref{thm:lognEta} are similar to that of Lemma \ref{lemma:OPT_Wlowerbound} and Theorem \ref{thm:logn} and we include them in the Appendix.
\begin{lemma}\label{lemma:OPT_WlowerboundEta}
Let $F$ be an \emph{admissible} function and $\tau$ is the first test chosen by {\textsc{GreedyTree}} on the set of objects $S$, then
\begin{equation*}
c(\tau)(F(S)-\delta)/(F(S)-F(S^q_\tau)) \leq OPT_{F:\delta}(S).
\end{equation*}
\end{lemma}

\begin{theorem} \label{thm:lognEta}
 {\textsc{GreedyTree}} constructs a decision tree achieving $O(\log n)$-factor approximation of the optimal max-cost in $D_{F:\delta}$ on the set $S$ of $n$ objects if  $F$ is \emph{admissible}.
\end{theorem}

\paragraph{Hinged-Pairs:}
Similar to early-stopping, we can also use the hinged-Pairs $P_{\alpha}$ \eqref{eq:hingedPairs} with $\alpha>0$ in {\textsc{GreedyTree}} to allow error-cost trade-off.
We first establish an error upper bound for trees in $D_{P_{\alpha}:0}$. 
\begin{lemma}\label{lemma8}
For a multi-class input set $S$ with $k$ classes, the classification error of any tree in $D_{P_{\alpha}:0}$ with $l$ leaves is bounded by $k(k-1)l\epsilon$, where we set $\alpha=\epsilon n$.
\end{lemma}
\begin{proof}
Suppose $j$ is the largest class in leaf $L$.
For $i\neq j$, if $n_L^{i}>\alpha$, we have $\max(n_L^{i}n_L^{j}-\alpha (n_L^{i}+n_L^{j}),0)=0$, which implies $n_L^{i}n_L^{j}\leq \alpha n_L$. So
\begin{equation*}
n_L^{i}\leq \frac{kn_L^{i}n_L^{j}}{n_L}\leq k\alpha = k\epsilon n.
\end{equation*}
If $n_L^{i}\leq \alpha$, we have $n_L^{i}\leq \epsilon n\leq k \epsilon n$.
So for any leaf $L$ we have $\frac{\sum_{i\neq j}n_L^{i}}{n}\leq k(k-1) \epsilon$. The overall error bound thus follows.
\end{proof}

Often in practice a tree may contain a relatively large number of leaves but only a small fraction of them contain most of the objects. A more refined upper bound on the error is given by the following lemma, which we prove in the Appendix.
\begin{lemma}\label{lemma9}
Consider a multi-class input set $S$ with $k$ classes and $\alpha=\epsilon n$. For any tree $T\in D_{P_\alpha:0}$ with $l$ leaves, given any $\eta\in [0,1]$, let $l_\eta$ be the smallest integer such that the largest $l_\eta$ leaves of $T$ have more than $1-\eta$ of the total number of objects $n$. Then the classification error is bounded by $k(k-1)l_\eta \epsilon + \frac{k-1}{k}\eta$.
\end{lemma}

Denote $D_{E:\epsilon}$ as the class of trees with classification error less than or equal to $\epsilon$ on the set of input $S$. We can further derive a useful relation between $D_{E:\epsilon}$ and $D_{P_{\epsilon n}:0}$.
\begin{lemma}\label{lemma:embedding}
For any multi-class input set $S$ with $k$ classes, $D_{E:\epsilon}\subseteq D_{P_{\epsilon n}:0}\subseteq D_{E:k(k-1)\epsilon l}$.
\end{lemma}
\begin{proof}
To show $D_{E:\epsilon}\subseteq D_{P_{\epsilon n}}$, for any tree $T \in D_{E:\epsilon}$, we have $\sum_{i=1}^{l} \tilde{n}_{Li}\leq \epsilon n$, where $l$ is the number of leaves and $\tilde{n}_{Li}$ is the number of objects in leaf $L_i$ that are not from the majority class: $\tilde{n}_{L}=n_L-n^{max}_{L}$. This implies $\tilde{n}_{L}\leq \epsilon n$ for all leaves of $T$. Suppose $j$ is the class with most number of objects in leaf $L$: $n_L^j=n_L^{max}$. It is not hard to see for any class $i\neq j$
\begin{equation*}
\frac{n_L^i n_L^j}{n_L^i+n_L^j}\leq n_L^i \leq \epsilon n,
\end{equation*}
which implies $[[n^i_L-\alpha]_+[n^j_L-\alpha]_+-\alpha^2]_+ =0$. Thus we have $F(L)=\sum_{p\neq q}[[n^p_L-\alpha]_+[n^q_L-\alpha]_+-\alpha^2]_+=0$. Thus $D_{E:\epsilon}\subseteq D_{P_{\epsilon n}:0}$.
$D_{P_{\epsilon n}:0}\subseteq D_{E:k(k-1)\epsilon l}$ follows from Lemma \ref{lemma8}.
\end{proof}

The main theorem of this section is the following.
\begin{theorem}
In multi-class classification with $k$ classes, if $T$ is the decision tree returned by {\textsc{GreedyTree}} using hinged-Pairs (setting $\alpha=\epsilon n$) applied on the set $S$ of $n$ objects, then we have the following:
\begin{multline*}
Cost_{P_\alpha}(S)\leq O(\log n) OPT_{P_\alpha:0}(S) \\
\leq O(\log n) OPT_{E:\epsilon}(S).
\end{multline*}
\end{theorem}
\begin{proof}
The first inequality follows from Theorem \ref{thm:logn} and the second inequality follows from Lemma \ref{lemma:embedding}.
\end{proof}
The above theorem states that for a given error parameter $\epsilon$, a greedy tree can be constructed using hinged-Pairs $P_{\alpha}$ by setting $\alpha=\epsilon n$, with the max-cost guaranteed to be within an $O(\log n)$ factor of the best possible max-cost among all decision trees that have classification error less than or equal to $\epsilon$. To our knowledge this is the first bound relating classification error to cost, which provides a theoretical basis for accuracy-cost trade-off.

\section{Experimental Results}\label{sec:experiments}

\begin{figure}
\centering
\includegraphics[trim=1.7cm 6cm 1.7cm 4.6cm,clip,width=0.45\textwidth]{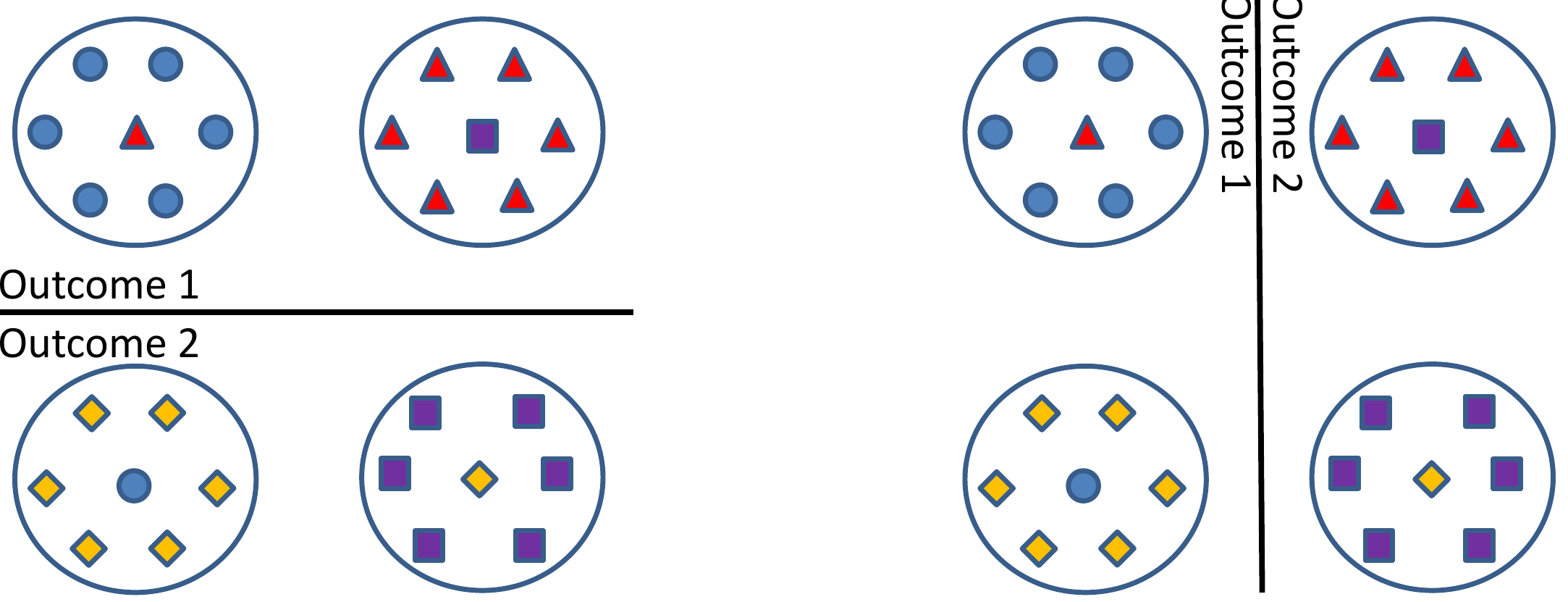}
\caption{A synthetic example to show the effect of outliers affecting max-cost. The left and right figures above show the test outcomes of test $t_1$ and $t_2$, respectively. } \label{fig:synth}
\end{figure}

\begin{figure}
\centering
\includegraphics[trim=5.75cm 4cm 6cm 4cm,clip,width=0.45\textwidth]{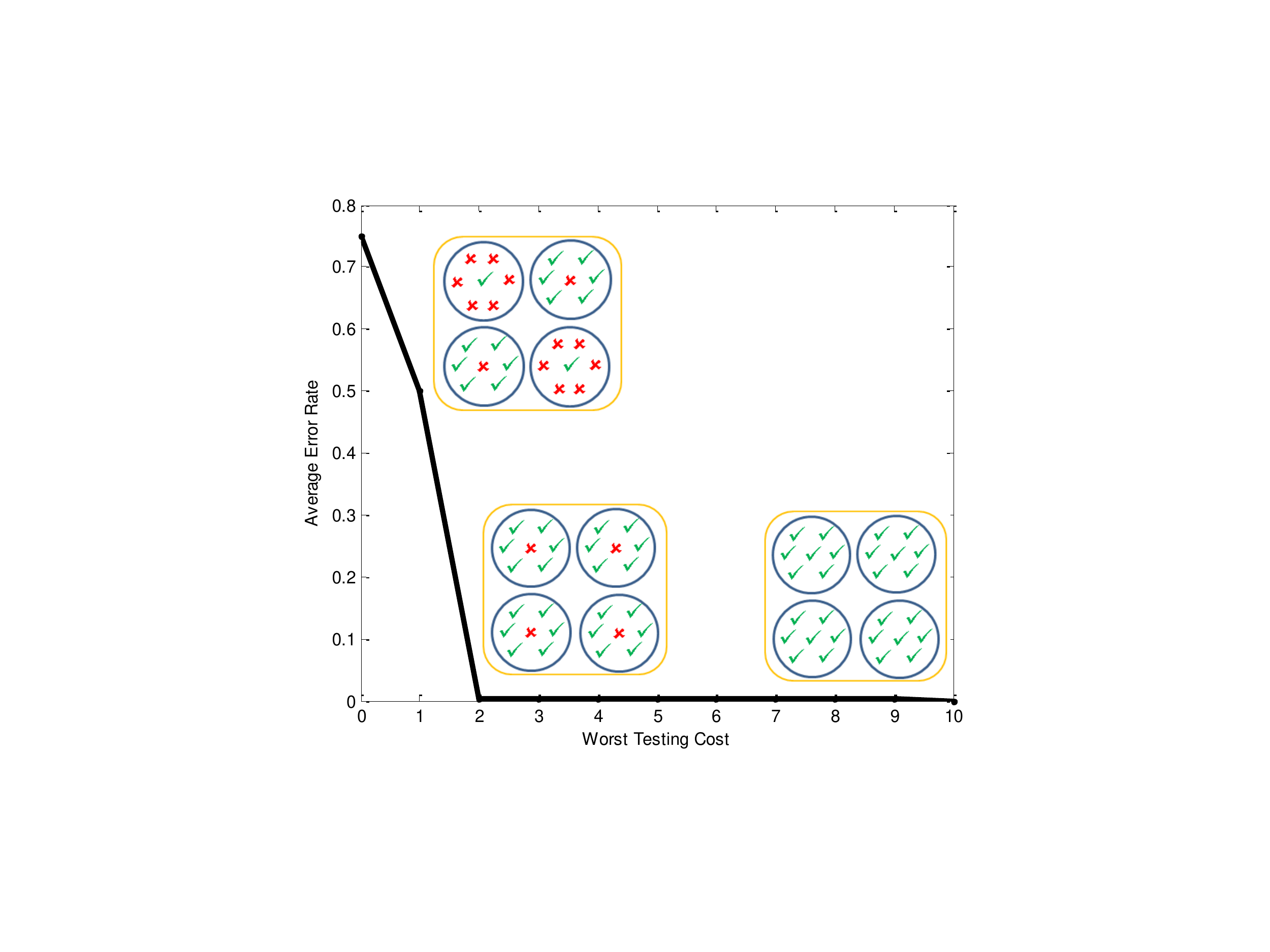}
\vspace{-.35cm}
\caption{The error-cost trade-off plot of the Algorithm \ref*{algo:GreedyTree} using Pairs on the synthetic example. $0.39\%$ error can be achieved using only a depth-2 tree but it takes a depth-10 tree to achieve zero error. } \label{fig:synthplot}
\end{figure}

\begin{figure*}[htb!]
\centering
\subfigure[House Votes]{\includegraphics[trim=40mm 88mm 40mm 90mm,clip,height=.23\linewidth]{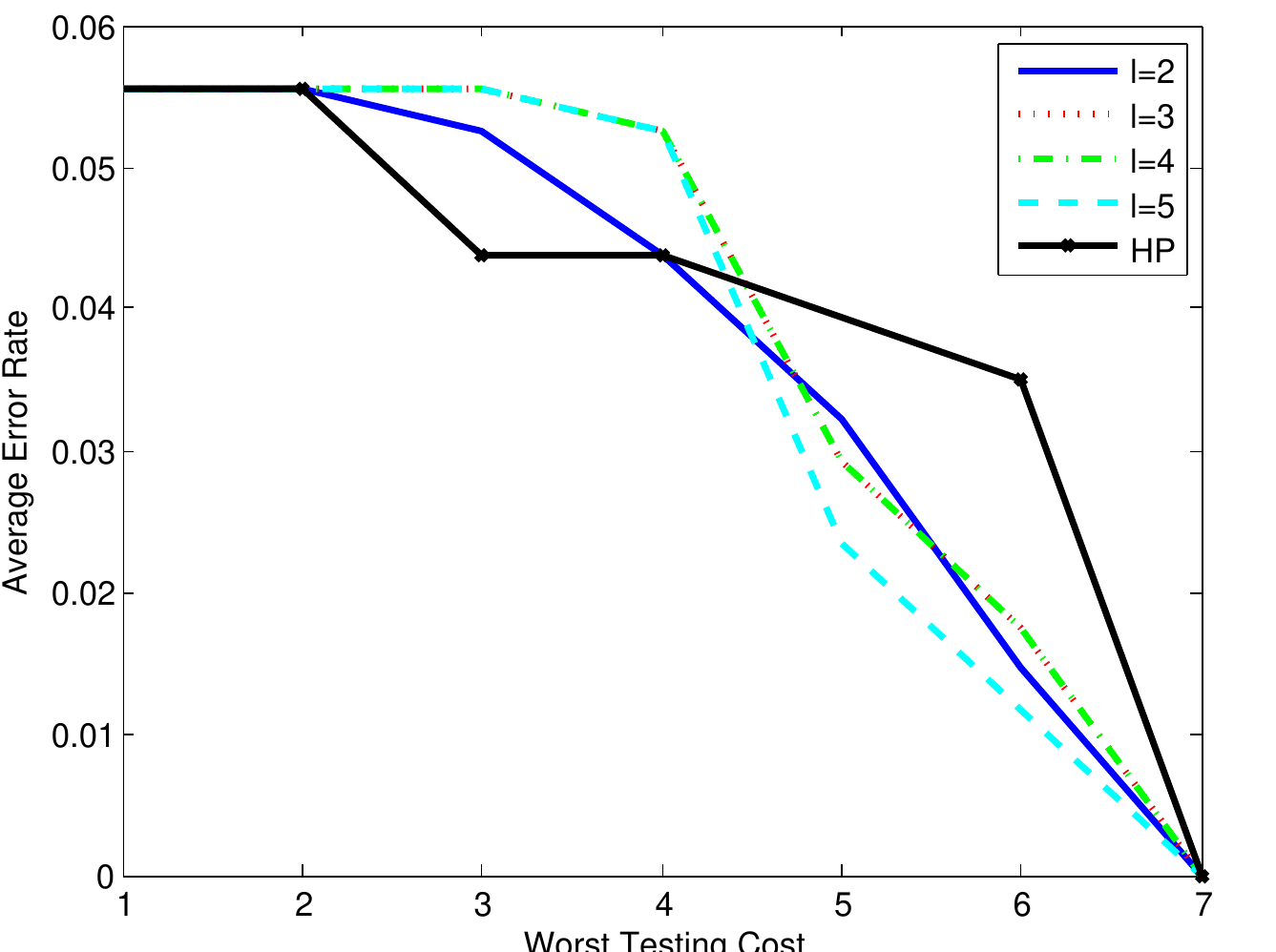}}
\subfigure[Sonar]{\includegraphics[trim=40mm 88mm 40mm 90mm,clip,height=.23\linewidth]{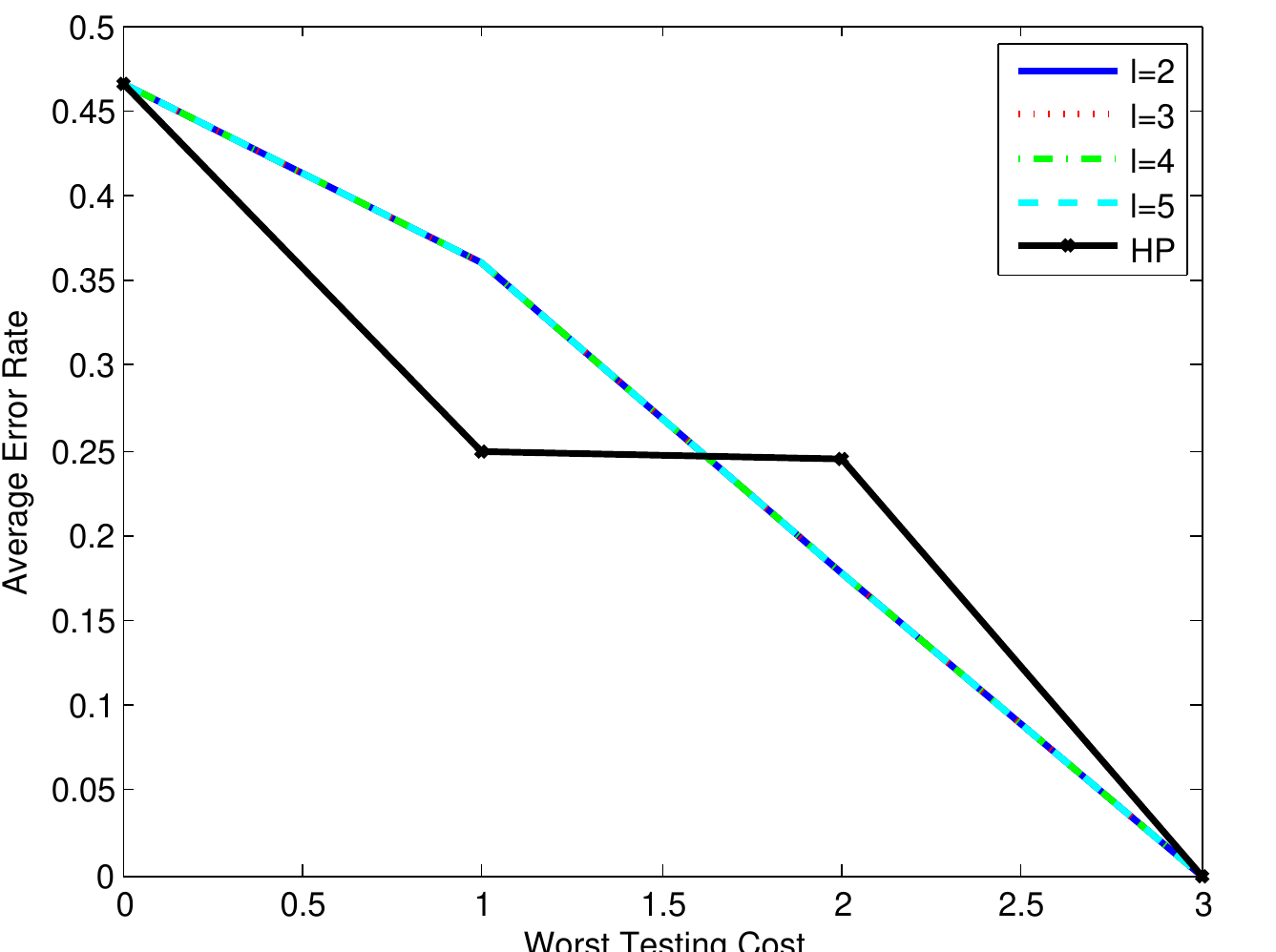}}
\subfigure[Ionosphere]{\includegraphics[trim=40mm 88mm 40mm 90mm,clip,height=.23\linewidth]{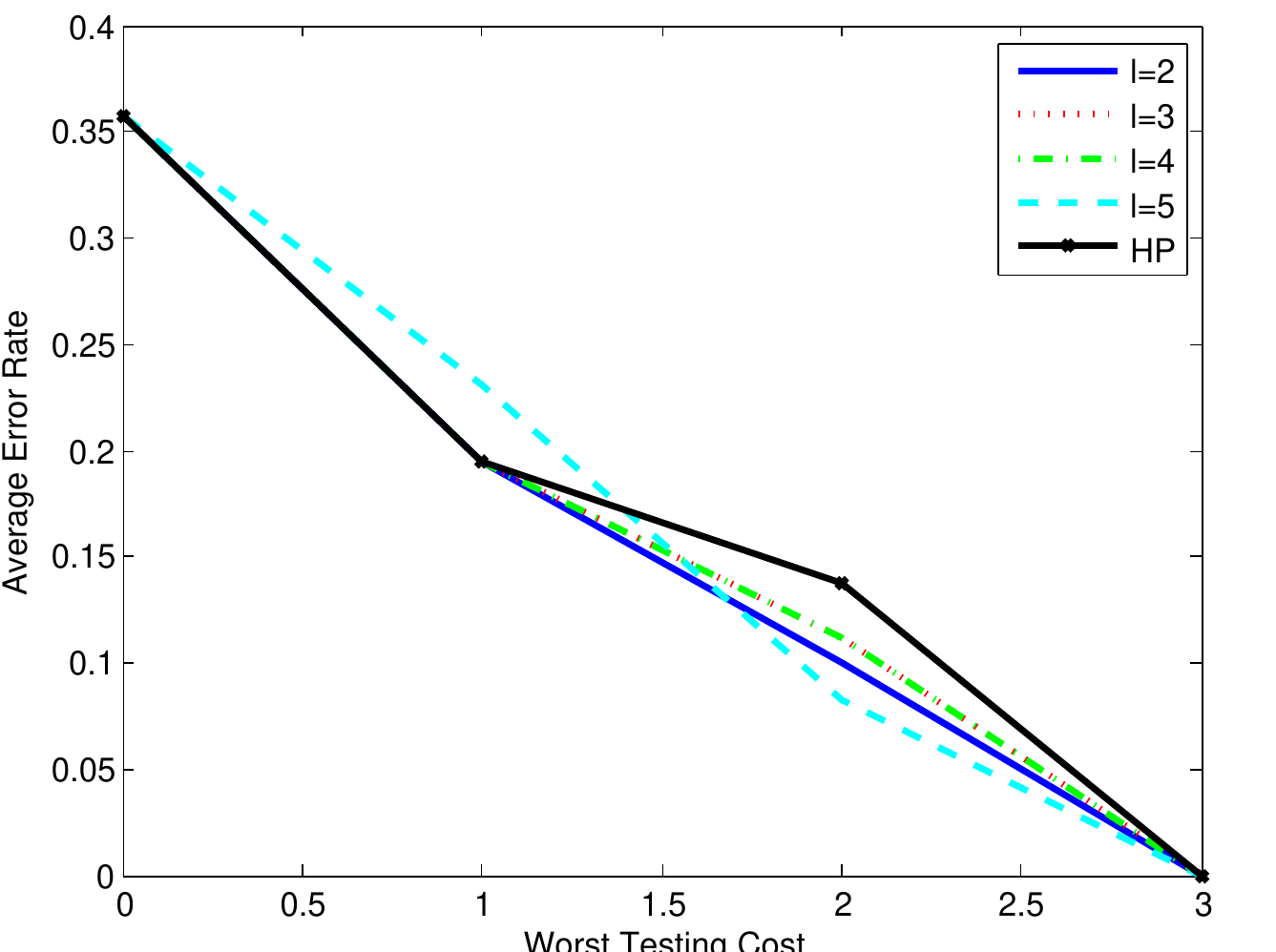}}\\[-2ex]
\subfigure[Statlog DNA]{\includegraphics[trim=40mm 88mm 40mm 90mm,clip,height=.23\linewidth]{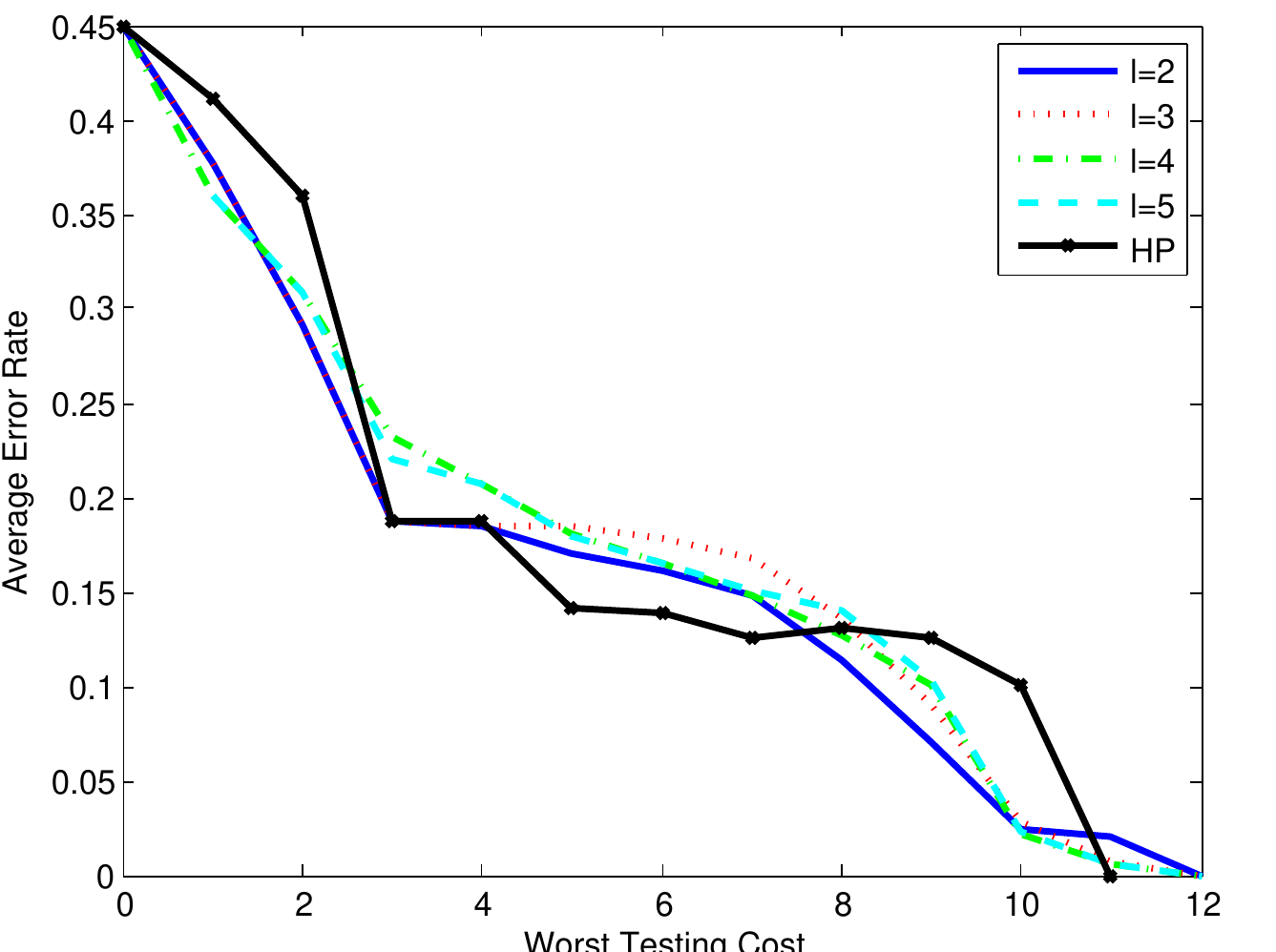}}
\subfigure[Boston Housing]{\includegraphics[trim=40mm 88mm 40mm 90mm,clip,height=.23\linewidth]{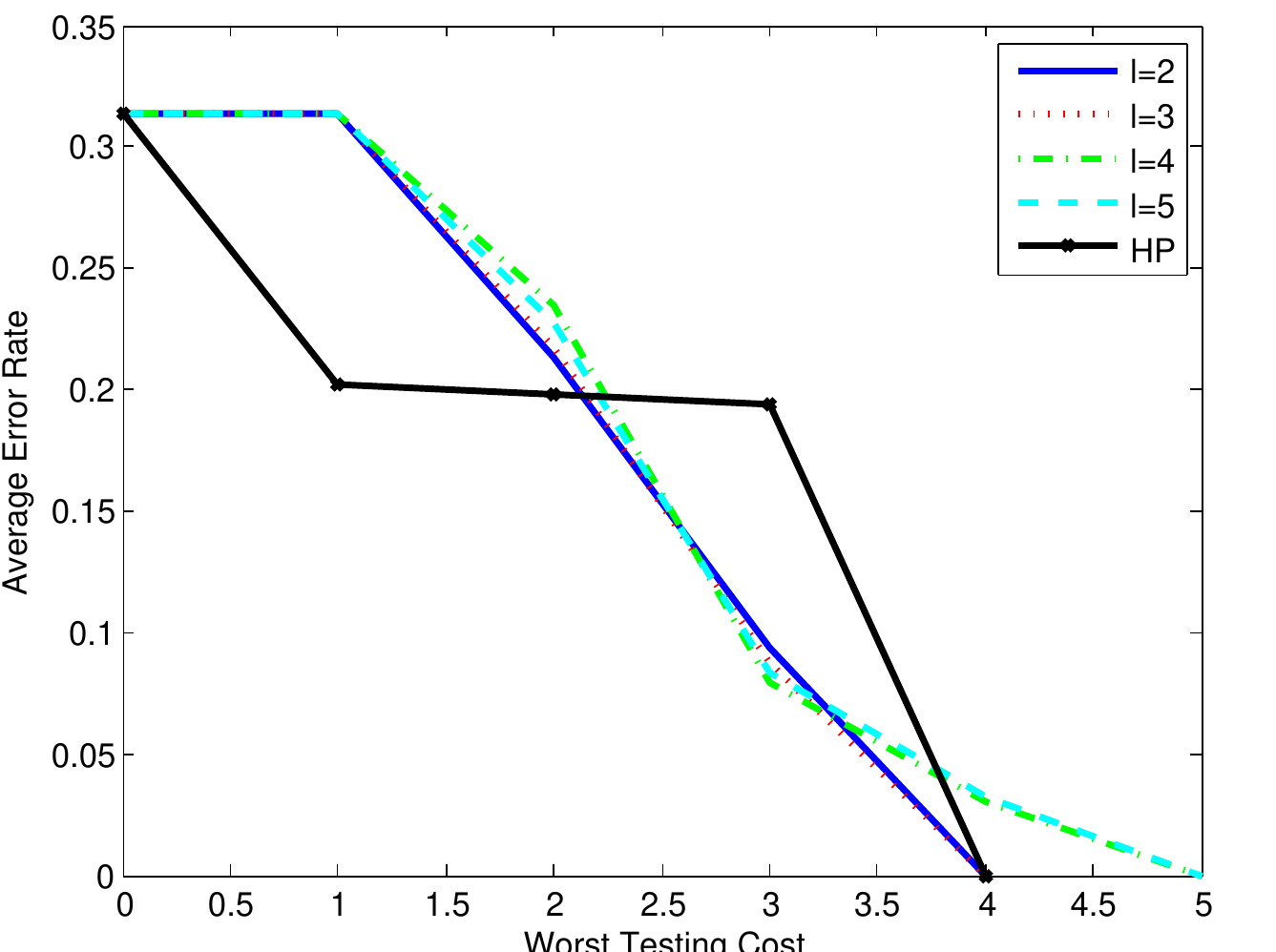}}
\subfigure[Soybean]{\includegraphics[trim=40mm 88mm 40mm 90mm,clip,height=.23\linewidth]{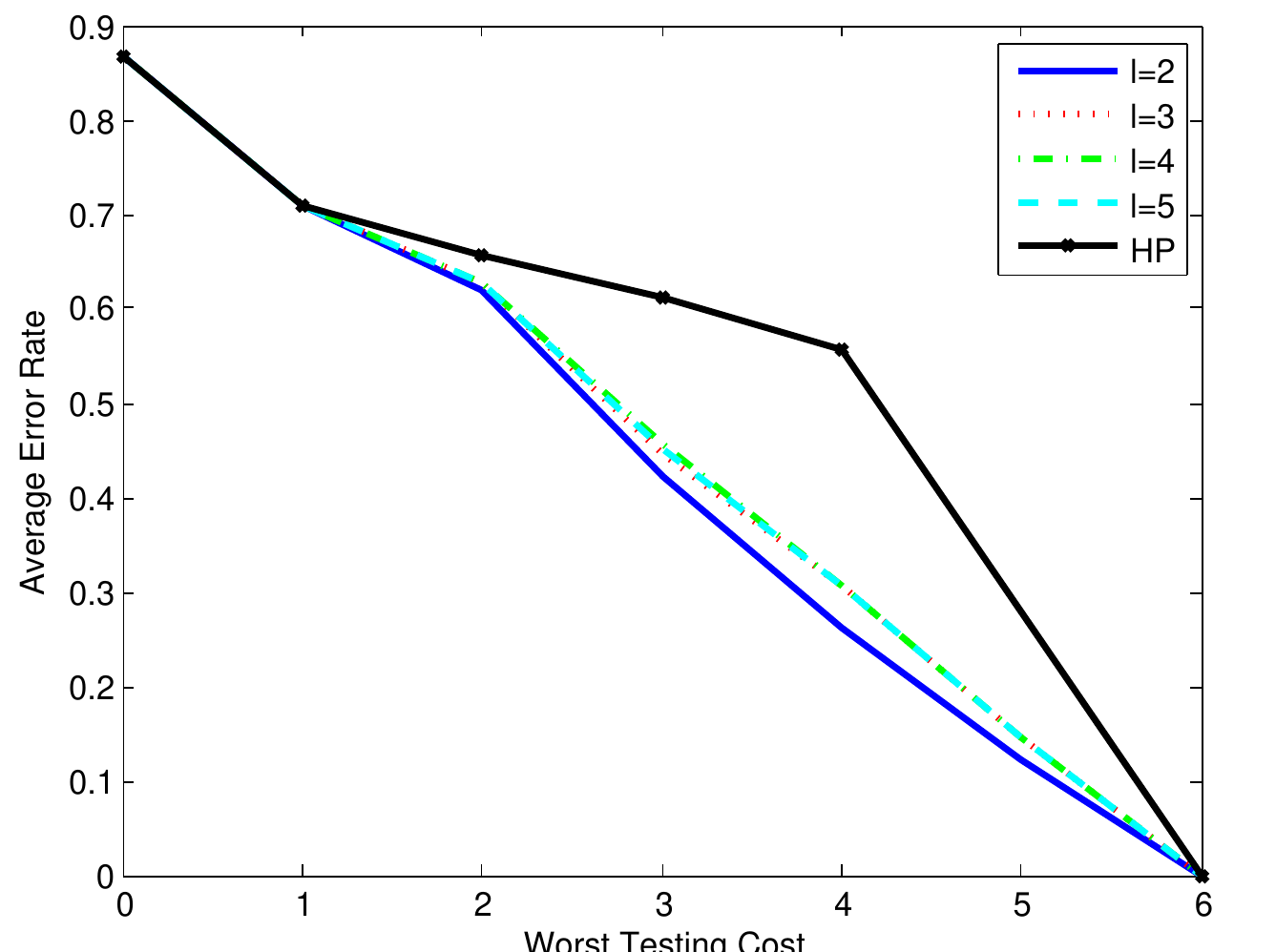}}\\[-2ex]
\subfigure[Pima]{\includegraphics[trim=40mm 88mm 40mm 90mm,clip,height=.23\linewidth]{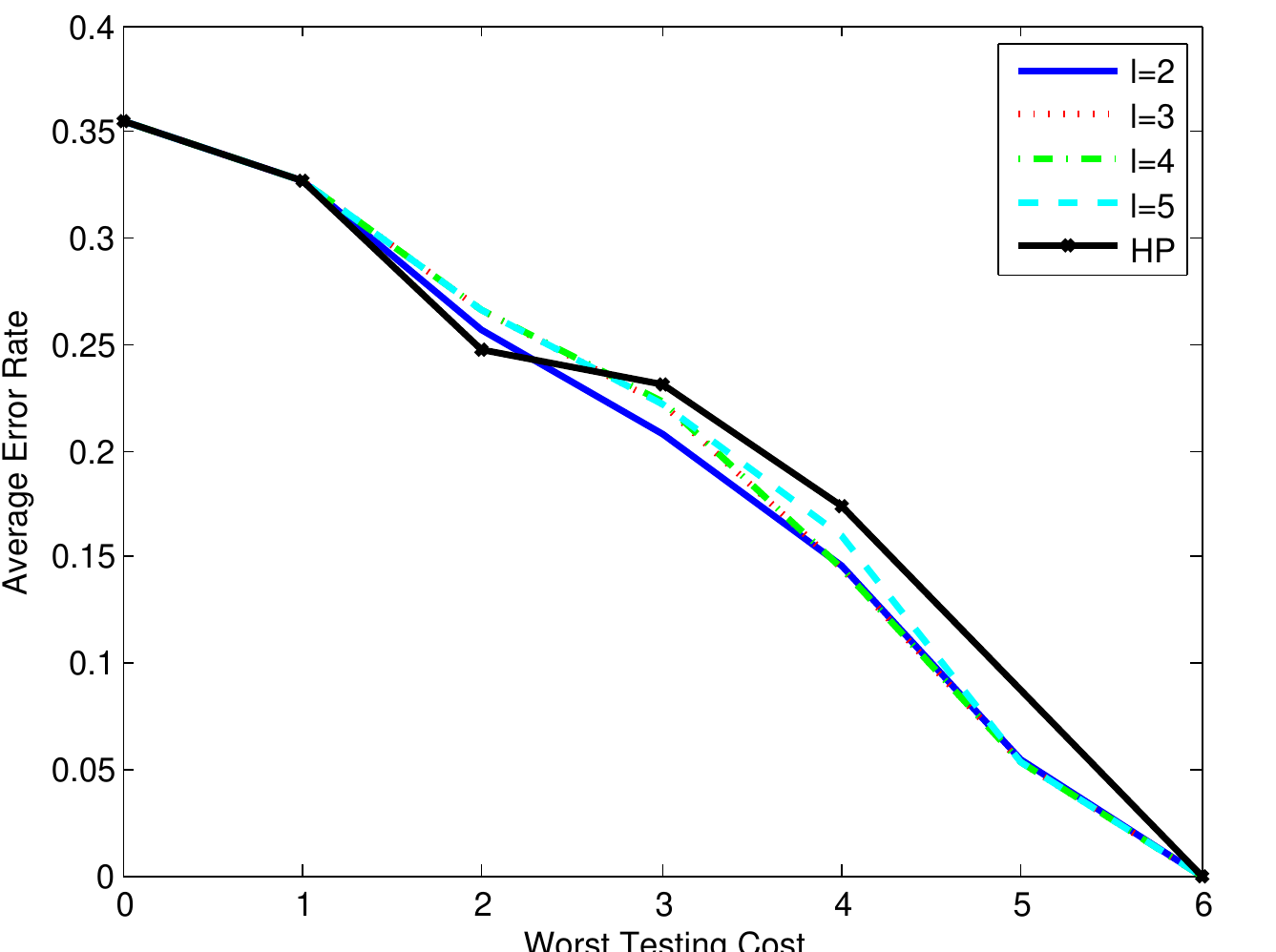}}
\subfigure[WBCD]{\includegraphics[trim=40mm 88mm 40mm 90mm,clip,height=.23\linewidth]{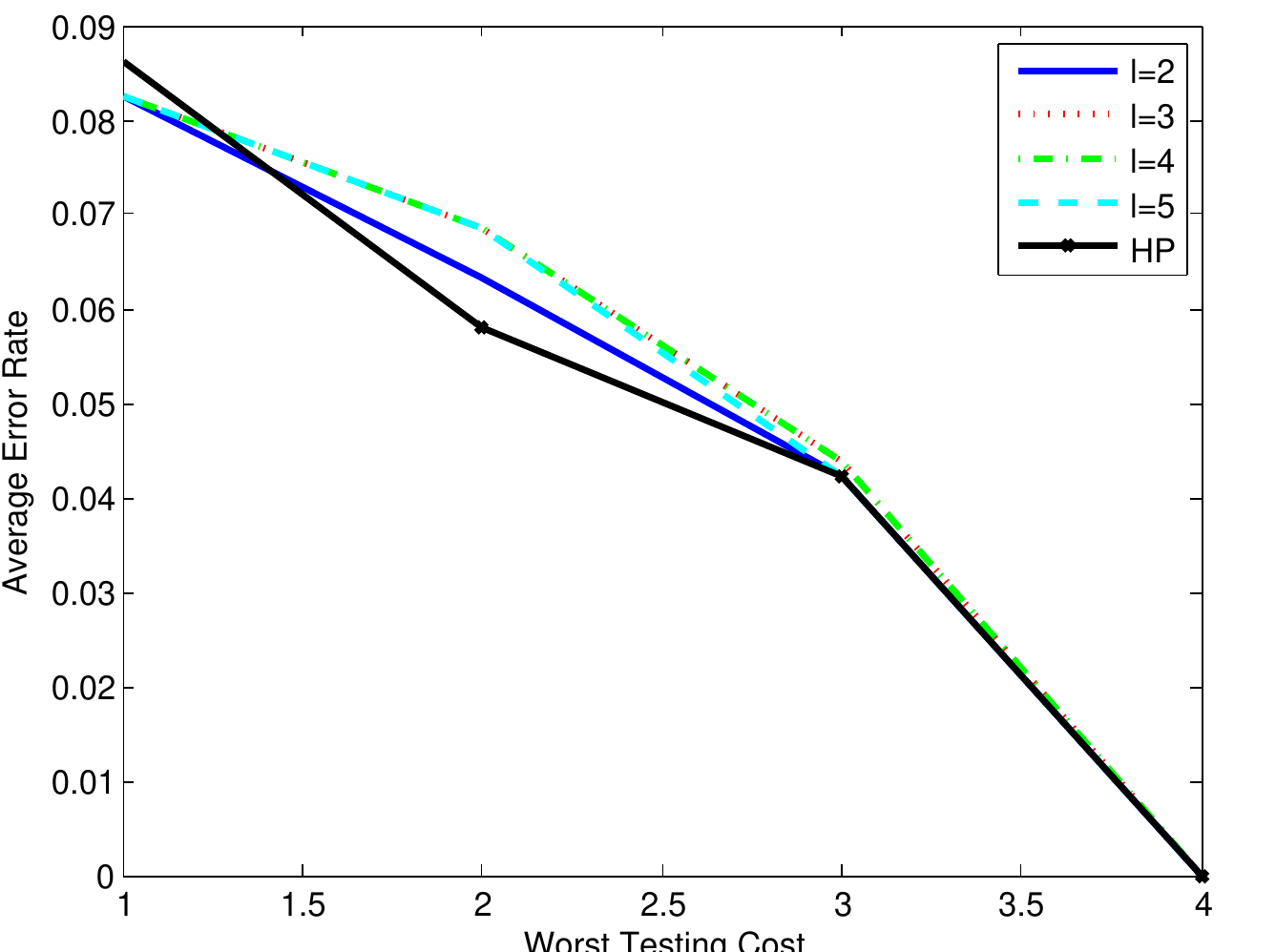}}
\subfigure[Mammography]{\includegraphics[trim=40mm 88mm 40mm 90mm,clip,height=.23\linewidth]{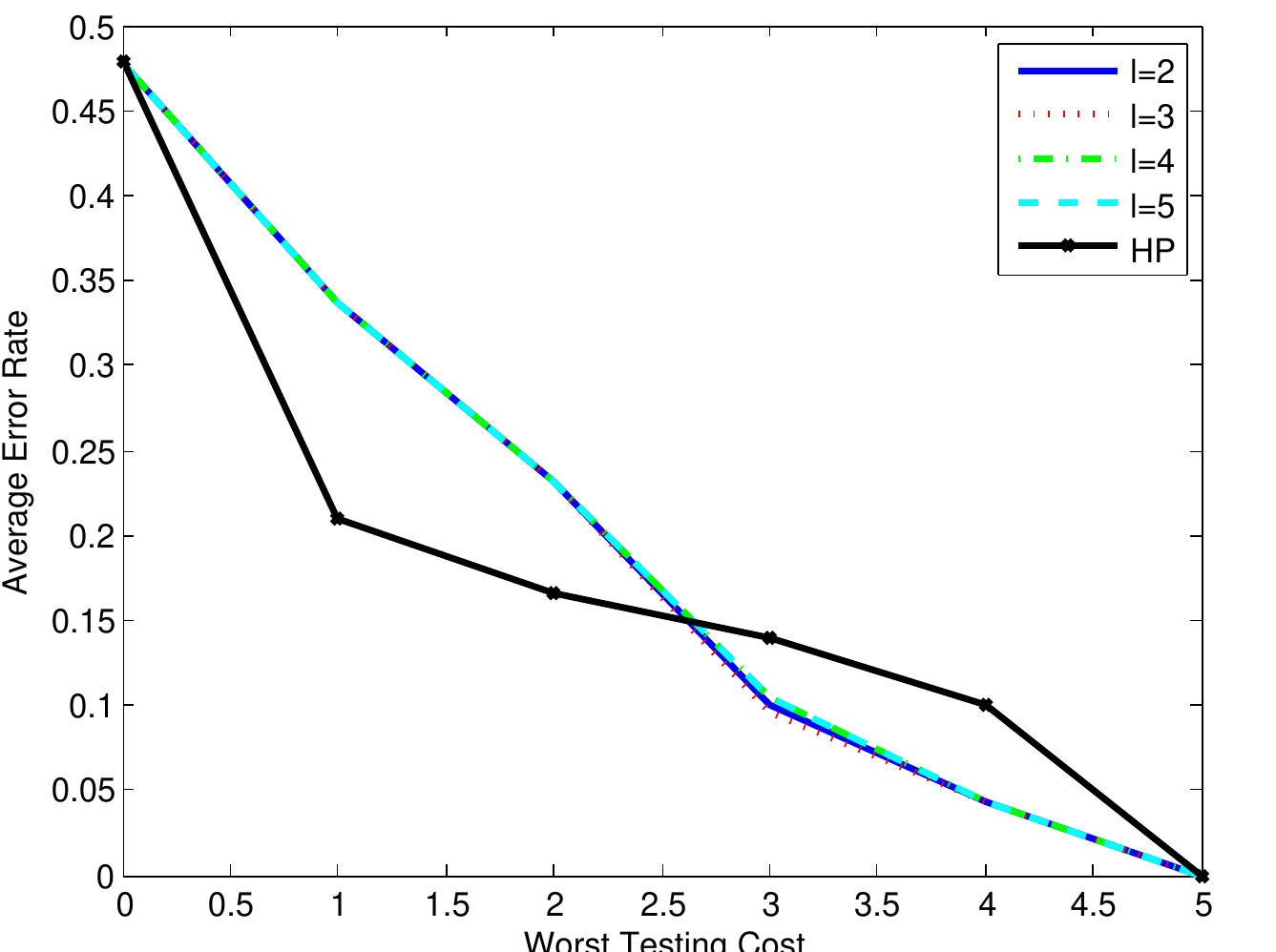}}
\vspace{-.35cm}
\caption{Comparison of classification error vs. max-cost for the Powers impurity function in \eqref{eq:powerFunc} for $l=2,3,4,5$ and the hinged-Pairs impurity function in \eqref{eq:hingedPairs}. Note that for both House Votes and WBCD, the depth $0$ tree is not included as the error decreases dramatically using a single test. In many cases, the hinged pairs impurity function outperforms the Powers impurity functions for trees with smaller max-costs, whereas the Powers impurity function outperforms the hinged-Pairs function for larger max-costs.}
\label{fig:real_world_tradeoff_curves}
\end{figure*}

We first demonstrate the effect of outliers using a simple synthetic example, where a small set of outliers dramatically increases the max-cost of the tree. We show that allowing a small number of errors in the tree drastically reduces the cost of the tree, allowing for efficient trees to be constructed in the presence of outliers. Next, we demonstrate the ability to construct decision trees on real world data sets. We observe a similar behavior to the synthetic data set on many of these data sets, where allowing a small amount of error results in trees with significantly lower cost. Additionally, we see the effect of impurity function choice on performance of the trees. For all real datasets, we present performance of the Powers impurity function presented in Eq. \eqref{eq:powerFunc} with $l=2,3,4,5$ and error introduced by early stopping as well as the hinged-Pairs impurity function presented in Eq. \eqref{eq:hingedPairs} with error introduced by varying the parameter $\alpha$.




\textbf{Synthetic Example:}
Here we consider a multi-class classification example to demonstrate the effect a small set of objects can have on the max-cost of the tree. Consider a data set composed of 1024 objects belonging to 4 classes with 10 binary tests available. Assume that the set of tests is complete, that is no two objects have the same set of test outcomes. Note that by fixing the order of the tests, the set of test outcomes maps each object to an integer in the range $[0,1023]$. From this mapping, we give the objects in the ranges $[1,255]$ , $[257,511]$ , $[513,767]$, and  $[769,1023]$ the labels $1$, $2$, $3$, and $4$, respectively, and the objects $0$, $256$, $512$, and $768$ the labels $2$, $3$, $4$, and $1$, respectively (Figure \ref{fig:synth} shows the data projected to the first two tests). Suppose each test carries a unit cost. By Kraft's Inequality~\cite{cover}, the optimal max-cost in order to correctly classify every object is 10, however, using only $t_1$ and $t_2$ as selected by the greedy algorithm, leads to a correct classification of all but 4 objects, as shown in Figure \ref{fig:synthplot}. For this type of data set, a constant sized set of costs can change from a tree with a constant max-cost to a tree with a $\log n$ max-cost.

\textbf{Data Sets:} We compare performance using 9 data sets from the UCI Repository \cite{UCI_repository}. We assume that all tests (features) have a uniform cost. For each data set, we replace non-unique objects with a single instance using the most common label for the objects, allowing every data set to be complete (perfectly classified by the decision trees). Additionally, continuous features are transformed to discrete features by quantizing to 10 uniformly spaced levels.
More details on the data sets used can be found in the Appendix.



\textbf{Error vs. Cost Trade-Off:} Fig. \ref{fig:real_world_tradeoff_curves} shows the trade-off between classification error and max-cost, which suggest two key trends. First, it appears that many data sets, such as house votes, Statlog DNA, Wisconsin breast cancer, and mammography, can be classified with minimal error using few tests. Intuitively, this small error appears to correspond to a small subset of outlier objects which require a large number of tests to correctly classify while the majority of the data can be classified with a small number of tests. Second, empirical evidence suggests that the optimal choice of impurity function is dependent on the desired max-cost of the tree. For trees with a smaller budget (and therefore lower depth), the hinged-Pairs impurity function outperforms the Powers impurity function with early stopping, whereas for larger budget (and greater depth), the Powers impurity function outperforms hinged-Pairs. This matches our intuitive understanding of the impurity functions, as the Powers impurity function biases towards tests which evenly divide the data whereas hinged-Pairs puts more emphasis on classification performance. 

\section{Conclusion}
We characterize a broad class of admissible impurity functions that can be used in a greedy algorithm to yield $O(\log n)$ guarantees of the optimal max-cost. We give examples of such admissible functions and demonstrate that they have different empirical properties even though they all enjoy the $O(\log n)$ guarantee. We further design admissible functions to allow for accuracy-cost trade-off and provide a bound relating classification error to cost. Finally, through real world datasets we demonstrate that our algorithm can indeed censor the outliers and achieve high classification accuracy using low max-cost. To visualize such outliers we construct a 2-D synthetic experiment and show our algorithm successfully identifies these as outliers.

\bibliographystyle{apalike}
\bibliography{OptimalTreeAISTATS}

\begin{thebibliography}{}

\bibitem[Bellala et~al., 2012]{GroupBasedActiveLearning}
Bellala, G., Bhavnani, S., and Scott, C. (2012).
\newblock Group-based active query selection for rapid diagnosis in
  time-critical situations.
\newblock {\em Information Theory, IEEE Transactions on}, 58(1):459--478.

\bibitem[Chakaravarthy et~al., 2011]{DecisionTreesforEntityIdentification}
Chakaravarthy, V.~T., Pandit, V., Roy, S., Awasthi, P., and Mohania, M.~K.
  (2011).
\newblock Decision trees for entity identification: Approximation algorithms
  and hardness results.
\newblock {\em ACM Trans. Algorithms}, 7(2):15:1--15:22.

\bibitem[Cicalese et~al., 2014]{DiagnosisDeterminationSimultaneous}
Cicalese, F., Laber, E.~S., and Saettler, A.~M. (2014).
\newblock Diagnosis determination: decision trees optimizing simultaneously
  worst and expected testing cost.
\newblock In {\em Proceedings of the 31th International Conference on Machine
  Learning, {ICML} 2014, Beijing, China, 21-26 June 2014}, volume~32 of {\em
  {JMLR} Proceedings}, pages 414--422. JMLR.org.

\bibitem[Cover and Thomas, 1991]{cover}
Cover, T.~M. and Thomas, J.~A. (1991).
\newblock {\em Elements of Information Theory}.
\newblock Wiley-Interscience, New York, NY, USA.

\bibitem[Dasgupta, 2004]{Dasgupta04analysisof}
Dasgupta, S. (2004).
\newblock Analysis of a greedy active learning strategy.
\newblock In {\em In Advances in Neural Information Processing Systems}, pages
  337--344. MIT Press.

\bibitem[Frank and Asuncion, 2010]{UCI_repository}
Frank, A. and Asuncion, A. (2010).
\newblock {UCI} machine learning repository.

\bibitem[Golovin and Krause, 2011]{AdaSubmodular_jair2011}
Golovin, D. and Krause, A. (2011).
\newblock Adaptive submodularity: Theory and applications in active learning
  and stochastic optimization.
\newblock {\em Journal of Artificial Intelligence Research (JAIR)},
  42:427--486.

\bibitem[Golovin et~al., 2010]{NearOptimalBayesianActiveLearning}
Golovin, D., Krause, A., and Ray, D. (2010).
\newblock Near-optimal bayesian active learning with noisy observations.
\newblock In Lafferty, J., Williams, C. K.~I., Shawe-Taylor, J., Zemel, R., and
  Culotta, A., editors, {\em Advances in Neural Information Processing Systems
  23}, pages 766--774.

\bibitem[Guillory and Bilmes, 2010]{InteractiveSubmodularSetCover_GuilloryB10}
Guillory, A. and Bilmes, J.~A. (2010).
\newblock Interactive submodular set cover.
\newblock In F√ºrnkranz, J. and Joachims, T., editors, {\em Proceedings of
  the 27th International Conference on Machine Learning (ICML-10)}, pages
  415--422. Omnipress.

\bibitem[Gupta et~al., 2010]{Gupta:OptimalDecisionTreesandAdaptiveTSP}
Gupta, A., Nagarajan, V., and Ravi, R. (2010).
\newblock Approximation algorithms for optimal decision trees and adaptive tsp
  problems.
\newblock In {\em Proceedings of the 37th International Colloquium Conference
  on Automata, Languages and Programming}, ICALP'10, pages 690--701, Berlin,
  Heidelberg. Springer-Verlag.

\bibitem[Hanneke, 2006]{Hanneke06thecost}
Hanneke, S. (2006).
\newblock The cost complexity of interactive learning.
\newblock unpublished.

\bibitem[Kosaraju et~al., 1999]{KosarajuOnanOptimalSplitTreeProblem}
Kosaraju, S.~R., Przytycka, T.~M., and Borgstrom, R.~S. (1999).
\newblock On an optimal split tree problem.
\newblock In {\em Proceedings of the 6th International Workshop on Algorithms
  and Data Structures}, WADS '99, pages 157--168, London, UK, UK.
  Springer-Verlag.

\bibitem[Moshkov,
  2010]{MoshkovGreedyAlgorithmwithWeightsforDecisionTreeConstruction}
Moshkov, M.~J. (2010).
\newblock Greedy algorithm with weights for decision tree construction.
\newblock {\em Fundam. Inf.}, 104(3):285--292.

\bibitem[Nowak, 2008]{Nowak08generalizedbinary}
Nowak, R. (2008).
\newblock Generalized binary search.
\newblock In {\em In Proceedings of the 46th Allerton Conference on
  Communications, Control, and Computing}, pages 568--574.

\bibitem[Saettler et~al., 2014]{TradingOffWorstExpectedCost}
Saettler, A., Laber, E., and Cicalese, F. (2014).
\newblock Trading off worst and expected cost in decision tree problems and a
  value dependent model.
\newblock {\em ArXiv}, pages 1--13.

\end{thebibliography}
\onecolumn
\section*{Appendix}
\paragraph{Proof of Lemma 3.4}
Before showing admissibility of the hinged-Pairs function in the multiclass setting, we first show  $P_\alpha(G)$ is \emph{admissible} for the binary setting.
\begin{lemma}\label{lemma:F_admissible_binary}
Consider the binary classification setting, let
\begin{equation*}
P_\alpha(G)=[[n^1_G-\alpha]_+[n^2_G-\alpha]_+-\alpha^2]_+,
\end{equation*}
where $[x]_+=\max(x,0)$. $P_\alpha(G)$ is \emph{admissible}.
\end{lemma}
\begin{proof}
All the properties are obviously true except supermodularity. To show supermodularity, suppose $R\subseteq G$ and object $j\notin R$. Suppose $j$ belongs to the first class. We need to show
\begin{equation}
P_\alpha(G\cup j)-P_\alpha(G)\geq P_\alpha(R\cup j) -P_\alpha(R). \label{eq:lemma5_1}
\end{equation}
Consider 3 cases:\\
(1) $P_\alpha(R)=P_\alpha(R\cup j)=0$: The right hand side of \eqref{eq:lemma5_1} is 0 and \eqref{eq:lemma5_1} holds because of monotonicity of $P_\alpha$.\\
(2) $P_\alpha(R)=0,P_\alpha(R\cup j)>0,P_\alpha(G)=0$: \eqref{eq:lemma5_1} reduces to $P_\alpha(G\cup j)\geq P_\alpha(R\cup j)$, which is true by monotonicity. \\
(3) $P_\alpha(R)=0,P_\alpha(R\cup j)>0,P_\alpha(G)>0$: Note that $P_\alpha(G)>0$ implies that $[n^1_G-\alpha]_+[n^2_G-\alpha]_+-\alpha^2>0$ which further implies $n_G^1>\alpha, n_G^2>\alpha$. Thus the left hand side is
\begin{equation*}
P_\alpha(G\cup j)-P_\alpha(G)=(n_G^1-\alpha+1)(n_G^2-\alpha)-\alpha^2-((n_G^1-\alpha)(n_G^2-\alpha)-\alpha^2)=n_G^2-\alpha.
\end{equation*}
The right hand side is
\begin{equation*}
P_\alpha(R\cup j)=(n^1_R-\alpha+1)(n^2_R-\alpha)-\alpha^2=(n^1_R-\alpha)(n^2_R-\alpha)-\alpha^2+(n^2_R-\alpha).
\end{equation*}
If $n_R^1\geq \alpha$, $ P_\alpha(R)=\max((n^1_R-\alpha)(n^2_R-\alpha)-\alpha^2,0)=0$ because $P_\alpha(R\cup j)>0$ implies $n_R^2>\alpha$. So $P_\alpha(R\cup j)\leq n^2_R-\alpha \leq n_G^2-\alpha = P_\alpha(G\cup j)-P_\alpha(G)$.\\
(4) $P_\alpha(R)>0$: We have
\begin{equation*}
P_\alpha(G\cup j)-P_\alpha(G)  = n_G^2-\alpha \geq n_R^2 - \alpha  = P_\alpha(R\cup j) -P_\alpha(R).
\end{equation*}
This completes the proof.
\end{proof}
Now we are ready to generalize from the binary hinged-Pairs function to the multiclass hinged-Pairs function.
Again, all properties are obviously except supermodularity. The supermodularity follows from the fact that each term in the sum is supermodular according to Lemma \ref{lemma:F_admissible_binary}.

\paragraph{Proof of Lemma 4.4}
We begin by considering any leaf $L$ of $T$, suppose $j$ is the largest class in $L$.
For $i\neq j$, if $n_L^{i}>\alpha$, we have
\begin{align*}
&[[n^i_L-\alpha]_+[n^j_L-\alpha]_+-\alpha^2]_+ \\
=&\max(n_L^{i}n_L^{j}-\alpha (n_L^{i}+n_L^{j}),0)=0
\end{align*}
, which implies $n_L^{i}n_L^{j}\leq \alpha n_L$. So
\begin{equation*}
n_L^{i}\leq \frac{kn_L^{i}n_L^{j}}{n_L}\leq k\alpha = k\epsilon n.
\end{equation*}
If $n_L^{i}\leq \alpha$, we have $n_L^{i}\leq \epsilon n\leq k \epsilon n$. Let $\tilde{n}_{L}$ be the number of objects in leaf $L$ that are not from the majority class: $\tilde{n}_{L}=n_L-n^{j}_{L}$.
So for any leaf $L$ we have $\frac{\tilde{n}_L}{n}=\frac{\sum_{i\neq j}n_L^{i}}{n}\leq k(k-1) \epsilon$.

Now we enumerate the leaves of $T$ in non-increasing order according to the number of objects they contain. Let $A$ be the set of the first $l_\eta$ leaves. By definition of $l_\eta$, the total number of objects contained in $A$ is ${n_A}\geq (1-\eta)n$.

The overall error bound is obtained by considering leaves in $A$ and the complement $\bar{A}$ separately:
\begin{align*}
\frac{\sum_{L\in A} \tilde{n}_L + \sum_{L\in \bar{A}} \tilde{n}_L}{n}& \leq \frac{k(k-1)\epsilon l_\eta n+ \frac{k-1}{k}\eta n}{n} \\
&= k(k-1)l_\eta \epsilon + \frac{k-1}{k}\eta,
\end{align*}
where we have used the fact that $\tilde{n}_L \leq \frac{k-1}{k} n_L$ and that $\sum_{L\in \bar{A}}n_L \leq \eta n$.

\paragraph{Details of Computation in Figure 1}
If Pairs is used, we can compute impurity of each set of interest: $P(G)=30\times 30=900, P(G_{t_1}^1)=30\times 10=300, P(G_{t_1}^2)=0, P(G_{t_2}^1)=P(G_{t_2}^2)=15\times 15=225$; according to Algorithm \ref{algo:GreedyTree}, we can compute $R(t_1)=\max \{\frac{1}{900-300},\frac{1}{900-0}\}=\frac{1}{600}, R(t_2)=\max \{\frac{1}{900-225},\frac{1}{900-225}=\frac{1}{675}\}$ so $t_2$ will be chosen. On the other hand, the impurities for the hinged-Pairs with $\alpha=8$ are $P_\alpha(G)=22\times 22=484, P_\alpha(G_{t_1}^1)=22\times 2=44, P_\alpha(G_{t_1}^2)=0, P_\alpha(G_{t_2}^1)=P_\alpha(G_{t_2}^2)=7\times 7=49$; again we can compute $R(t_1)=\max \{\frac{1}{484-44},\frac{1}{484-0}\}=\frac{1}{440}, R(t_2)=\max \{\frac{1}{484-49},\frac{1}{484-49}=\frac{1}{435}\}$ so $t_1$ will be chosen. The above example shows that Pairs has a stronger preference to balanced tests and may in some cases lead to poor classification result.

\paragraph{Details of Data Sets}
 The house votes data set is composed of the voting records for 435 members of the U.S. House of Representatives (342 unique voting records) on 16 measures, with a goal of identifying the party of each member. The sonar data set contains 208 sonar signatures, each composed of energy levels (quantized to 10 levels) in 60 different frequency bands, with a goal of identifying  The ionosphere data set has 351 (350 unique) radar returns, each composed of 34 responses (quantized to 10 levels), with a goal of identifying if an event represents a free electron in the ionosphere. The Statlog DNA data set is composed of 3186 (3001 unique) DNA sequences with 180 features, with a goal of predicting whether the sequence represents a boundary of DNA to be spliced in or out. The Boston housing data set contains 13 attributes (quantized to 10 levels) pertaining to 506 (469 unique) different neighborhoods around Boston, with a goal of predicting which quartile the median income of the neighborhood the neighborhood falls. The soybean data set is composed of 307 examples (303 unique) composed of 34 categorical features, with a goal of predicting from among 19 diseases which is afflicting the soy bean plant. The pima data set is composed of 8 features (with continuous features quantized to 10 levels) corresponding to medical information and tests for 768 patients (753 unique feature patterns), with a goal of diagnosing diabetes. The Wisconsin breast cancer data set contains 30 features corresponding to properties of a cell nucleus for 569 samples, with a goal of identifying if the cell is malignant or benign. The mammography data set contains 6 features from mammography scans (with age quantized into 10 bins) for 830 patients, with a goal of classifying the lesions as malignant or benign.
\end{document}